\documentclass[letterpaper, 10 pt, conference]{ieeeconf}  

\IEEEoverridecommandlockouts                              

\usepackage{amsmath} 
\usepackage{amssymb}  
\usepackage{amsfonts}
\usepackage{algorithm}
\usepackage{algpseudocode}
\usepackage{xcolor}
\usepackage{soul}
\usepackage[switch]{lineno} 
\usepackage{url}
\usepackage{amsmath}
\usepackage{mathtools} 
\usepackage[capitalise]{cleveref}
\usepackage{cite}
\usepackage{graphicx}
\usepackage{cuted}
\usepackage{capt-of}
\usepackage{booktabs}
\usepackage{pifont}
\definecolor{mygreen}{RGB}{0,150,0}
\definecolor{myred}{RGB}{200,0,0}
\newcommand{\cmark}{\textcolor{mygreen}{\ding{51}}}
\newcommand{\xmark}{\textcolor{myred}{\ding{55}}}

\newtheorem{assumption}{\bf{Assumption}}

\newtheorem{lemma}{\bf{Lemma}}

\newtheorem{problem}{\bf{Problem}}

\xdefinecolor{redZ}{RGB}{234,29,93}
\xdefinecolor{blueZ}{RGB}{19,106,213}
\xdefinecolor{yellowZ}{RGB}{251,138,46}
\xdefinecolor{greenYL}{RGB}{0, 191, 0}

\title{\LARGE \bf
Safe Consensus of Cooperative Manipulation with Hierarchical Event-Triggered Control Barrier Functions
}

\author{Simiao Zhuang, Bingkun Huang, Zewen Yang$^{\dagger}$
\thanks{S. Zhuang, B. Huang, and Z. Yang are with Munich Institute of Robotics and Machine Intelligence (MIRMI), Technical University of Munich (TUM), 80992 Munich, Germany.}
\thanks{
$^{\dagger}$Corresponding Author. \textless zewen.yang@tum.de \textgreater}
}

\begin{document}

\maketitle
\thispagestyle{empty}
\pagestyle{empty}


\begin{abstract}
Cooperative transport and manipulation of heavy or bulky payloads by multiple manipulators requires coordinated formation tracking, while simultaneously enforcing strict safety constraints in varying environments with limited communication and real-time computation budgets. 
This paper presents a distributed control framework that achieves consensus coordination with safety guarantees via hierarchical event-triggered control barrier functions (CBFs). 
We first develop a consensus-based protocol that relies solely on local neighbor information to enforce both translational and rotational consistency in task space. 
Building on this coordination layer, we propose a three-level hierarchical event-triggered safety architecture with CBFs, which is integrated with a risk-aware leader selection and smooth switching strategy to reduce online computation. 
The proposed approach is validated through real-world hardware experiments using two Franka manipulators operating with static obstacles, as well as comprehensive simulations demonstrating scalable multi-arm cooperation with dynamic obstacles. 
Results demonstrate higher precision cooperation under strict safety constraints, achieving substantially reduced computational cost and communication frequency compared to baseline methods. 
Project page is available at \url{https://zsmandreas.github.io/Safe-Consensus-HET-CBF}.
\end{abstract}

\section{Introduction}
Multi-agent systems (MASs) enable robots to accomplish tasks that are difficult or infeasible for a single platform, for instance, when payload weight or size exceeds individual capabilities~\cite{Olfati_PIEEE_2027_Consensus}. 
In cooperative transport and manipulation, multiple manipulators act as agents physically coupled through a shared payload, yielding a closed-chain system: the team must track a common reference while maintaining prescribed relative end-effector (EE) poses to preserve formation stability~\cite{Dohmann_TRO_2020_Distributed}. Recent work has developed distributed coordination laws based on consensus and cooperative tracking, enabling team-level objectives using only local neighbor information~\cite{Yang_ACC_2024_Cooperative,Dai_TNNLS_2025_Cooperative}. 
However, robust performance remains challenging in practice because safety requirements in cluttered environments, limited communication, and tight real-time computation budgets jointly restrict feasibility and degrade robustness~\cite{Javier_IJRR_2017_Multi}. 

From a coordination viewpoint, cooperative transport and manipulation naturally map to formation objectives, whereas safety necessitates additional constraint enforcement. 
Each agent must preserve desired relative displacements and poses with respect to its neighbors while the team tracks a reference trajectory~\cite{Cos_RAL_2022_Adaptive,Farivarnejad_ARCRA_2022_Multirobot}. 
In human-in-the-loop cooperative manipulation, it becomes even more demanding, where a human operator specifies high-level intent (e.g., a motion direction or target), and the robotic team must autonomously execute the motion while preserving formation consistency and constrained behavior~\cite{MusicIROS2017,CaccavaleTMech2008}.
Many formation-control methods have been proposed, including virtual-structure methods, leader–follower architectures, and distributed consensus protocols~\cite{Ren_IET_2007_Consensus,Yan_OE_2020_Virtual,Yang_CDC_2021_Distributed}. 
Among these, consensus-based control is particularly attractive because it is local and scalable, enabling formation regulation using only neighbor information~\cite{LiuBucknallRobotica2018,RoyIROS2018,LeeChwaISR2018}. 
Yet, in safety-critical cooperative manipulation, tracking and formation maintenance alone are insufficient: the system must also satisfy strict constraints such as obstacle avoidance, inter-agent collision avoidance, and actuator limits under second-order dynamics. 
This motivates the need for a principled safety mechanism that integrates cleanly with distributed coordination~\cite{WangAmesEgerstedtTRO2017,JiangGuoTCNS2024}.

Control barrier functions (CBFs) have emerged as an effective tool for enforcing safety constraints by rendering a prescribed safe set forward invariant through appropriately designed control laws, often implemented via online safety filtering~\cite{morton2025safe}. 
For robotic manipulators and MASs with second-order dynamics, input-constrained CBFs provide a systematic way to impose safety requirements in the presence of explicit actuation limits with moving obstacles ~\cite{iccbf}. 
However, scaling CBF-based safety filtering to large-scale cooperative manipulation in MASs remains computationally demanding in real time~\cite{WangAmesEgerstedtTRO2017,luo2020multi}. 
As the number of agents and obstacles grows, the number of collision pairs and associated safety constraints increases rapidly, and solving quadratic problems (QPs) at every control step for all agent can become impractical. 
Especially in time-varying environments, frequent updates of obstacle information further amplify online computation. 
Therefore, it is necessary to develop a distributed, computationally efficient safety framework that simultaneously enforces formation consistency and collision avoidance for cooperative manipulation. 

To address these challenges, this paper presents a novel hierarchical event-triggered (HET) control framework. 
First, we design a distributed consensus-tracking protocol that relies only on local neighbor information to ensure consensus formation. 
Second, we develop a three-layer HET safety architecture with a risk-aware leader selection and smooth switching strategy, reducing online computation and communication overhead while preserving constraint satisfaction under second-order dynamics and actuation limits. 
Finally, the real-world hardware experiments and comprehensive simulation demonstrate superior performance compared to existing baseline methods.

\section{Preliminaries}
\label{sec_preliminaries}

\subsection{Graph Theory and Notation}\label{subsec_notation}
The communication topology of the MAS is modeled by an undirected graph $\mathcal{G} = (\mathcal{V}, \mathcal{E})$ over $N \in \mathbb{N}$ agents, where $\mathcal{V} = {1, \ldots, N}$ is the node set indexing the agents and $\mathcal{E} \subseteq \mathcal{V} \times \mathcal{V}$ is the edge set. 
An edge $(i, j) \in \mathcal{E}$ indicates that agents $i$ and $j$ exchange information bidirectionally. The neighbor set of agent $i$ is defined as $\mathcal{N}_i = {j \in \mathcal{V} : (i,j) \in \mathcal{E}}$. 
The adjacency matrix is $\mathcal{A} = [a_{ij}] \in \mathbb{R}^{N \times N}$, where $a_{ij} = a_{ji} = 1$ if $(i,j) \in \mathcal{E}$, and $a_{ij} = 0$ otherwise.
For each agent $i \in \mathcal{V}$, let $\boldsymbol{p}_i(\cdot) \in \mathbb{R}^3$ and $\boldsymbol{R}_i(\cdot) \in \mathrm{SO}(3)$ denote the EE position and orientation derived from the forward kinematics map.

\subsection{System Modeling and Problem Formulation}
To address the complexity of cooperative transport, we adopt a hierarchical control architecture that decouples the high-level consensus planning in the task space from the low-level nonlinear manipulator dynamics in the joint space.

\subsubsection{Manipulator Dynamics and Task-Space Formulation}
The dynamics of the $i$-th manipulator follows the standard rigid-body model
\begin{equation}
\boldsymbol{M}_i(\boldsymbol{q}_i)\ddot{\boldsymbol{q}}_i 
+ \boldsymbol{C}_i(\boldsymbol{q}_i,\dot{\boldsymbol{q}}_i)\dot{\boldsymbol{q}}_i 
+ \boldsymbol{g}_i(\boldsymbol{q}_i) = \boldsymbol{\tau}_i,
\label{eq:joint_dynamics}
\end{equation}
where $\boldsymbol{q}_i,\dot{\boldsymbol{q}}_i,\ddot{\boldsymbol{q}}_i\in\mathbb{R}^n$ denote the joint position, velocity, and acceleration, 
$\boldsymbol{M}_i$, $\boldsymbol{C}_i$, and $\boldsymbol{g}_i$ represent the inertia, Coriolis/centrifugal, and gravity terms, and $\boldsymbol{\tau}_i$ is the control torque.

The coordination objective is defined in task space. Let $\boldsymbol{x}_i=[\boldsymbol{p}_i^\top,\boldsymbol{\phi}_i^\top]^\top\in\mathbb{R}^6$ denote the end-effector (EE) pose, where $\boldsymbol{p}_i\in\mathbb{R}^3$ is the Cartesian position and $\boldsymbol{\phi}_i\in\mathbb{R}^3$ is a minimal orientation representation. The EE twist is $\boldsymbol{v}_i=[\dot{\boldsymbol{p}}_i^\top,\boldsymbol{\omega}_i^\top]^\top$.
The differential kinematics is
\begin{equation}
\boldsymbol{v}_i = \boldsymbol{J}_i(\boldsymbol{q}_i)\dot{\boldsymbol{q}}_i,
\label{eq:diff_kin}
\end{equation}
where $\boldsymbol{J}_i(\boldsymbol{q}_i)\in\mathbb{R}^{6\times n}$ is the geometric Jacobian. Differentiating \eqref{eq:diff_kin} gives
\begin{equation}
\dot{\boldsymbol{v}}_i = \boldsymbol{J}_i(\boldsymbol{q}_i)\ddot{\boldsymbol{q}}_i 
+ \dot{\boldsymbol{J}}_i(\boldsymbol{q}_i,\dot{\boldsymbol{q}}_i)\dot{\boldsymbol{q}}_i.
\label{eq:task_acc}
\end{equation}

\subsubsection{Hierarchical Control Design}
We propose a feedback linearization scheme to transform the nonlinear  manipulator dynamics into a linear double-integrator system in the task space for formation consensus control. 
Let $\boldsymbol{u}_i \in \mathbb{R}^6$ be the virtual control input to be designed by the consensus protocol. 
The reference joint acceleration $\ddot{\boldsymbol{q}}_{i}^{cmd}$ required to track the consensus protocol as 
\begin{equation}
    \ddot{\boldsymbol{q}}_{i}^{cmd} = \boldsymbol{J}_i^+(\boldsymbol{q}_i) 
    \!\big( \boldsymbol{u}_i - \dot{\boldsymbol{J}}_i(\boldsymbol{q}_i, 
    \dot{\boldsymbol{q}}_i)\dot{\boldsymbol{q}}_i \big) 
    + \boldsymbol{N}_i(\boldsymbol{q}_i)\boldsymbol{\eta}_i,
    \label{eq:inverse_acc}
\end{equation}
where $\boldsymbol{J}_i^\dagger = \boldsymbol{J}_i^\top(\boldsymbol{J}_i \boldsymbol{J}_i^\top)^{-1}$ denotes the right Moore-Penrose pseudoinverse, $\boldsymbol{N}_i = (\boldsymbol{I}_n - \boldsymbol{J}_i^\dagger \boldsymbol{J}_i) \in \mathbb{R}^{n \times n}$ is the null-space projector exploiting the  $(n-6)$-dimensional redundancy for secondary tasks (e.g., joint limit  avoidance, singularity escape), and $\boldsymbol{\eta}_i \in \mathbb{R}^n$  is an auxiliary joint acceleration vector. 
Substituting $\ddot{\boldsymbol{q}}_{i}^{cmd}$ into \eqref{eq:joint_dynamics}, the control torque is computed via the inverse dynamics law as follows
\begin{equation}
    \boldsymbol{\tau}_i = \boldsymbol{M}_i(\boldsymbol{q}_i)
    \ddot{\boldsymbol{q}}_{i}^{cmd} + \boldsymbol{C}_i(\boldsymbol{q}_i, 
    \dot{\boldsymbol{q}}_i)\dot{\boldsymbol{q}}_i + \boldsymbol{g}_i(\boldsymbol{q}_i).
    \label{eq:inverse_dynamics_ctrl}
\end{equation}
Substituting \eqref{eq:inverse_acc} into \eqref{eq:inverse_dynamics_ctrl} and considering the  full-row-rank property $\boldsymbol{J}_i\boldsymbol{J}_i^+ = \boldsymbol{I}_6$, it yields the canonical double-integrator dynamics in the task space as follows
\begin{equation}
    \dot{\boldsymbol{x}}_i = \boldsymbol{v}_i, \qquad
    \dot{\boldsymbol{v}}_i = \boldsymbol{u}_i.
    \label{eq:double_integrator_model}
\end{equation}
This formulation reduces the consensus protocol design to specifying $\boldsymbol{u}_i$, which is designed in \cref{subsec_consensus_tracking}.
The exact feedback linearization \eqref{eq:double_integrator_model} relies on $\boldsymbol{J}_i\boldsymbol{J}_i^{+} = \boldsymbol{I}_6$. In practice, a numerically robust approximation to $\boldsymbol{J}_i^{+}$ is employed near kinematic singularities, as detailed in Section~\ref{subsec_acc_mapping}.

\subsubsection{Coordination Objective}
In cooperative rigid-body transport, enforcing strict synchronization of both translational and rotational degrees of freedom imposes a fully constrained kinematic coupling among agents. 
However, manufacturing tolerances, calibration errors, or unmodeled components may introduce inevitable pose discrepancies at the grasp points, where unavoidable pose discrepancies give rise to undesirable internal wrenches that can damage the payload or destabilize the formation. 
To relax this problem, we propose a practical coordination task that enforces asymptotic positional consensus while relaxing the rotational requirement to a bounded compliance condition. 
The coordination objective is formalized as follows.
\begin{problem}[Practical Consensus Formation]
\label{prob_split_coordination}
Consider a linearized double-integrator MASs \eqref{eq:double_integrator_model} communicating over a connected undirected graph $\mathcal{G}$. 
The object is to design a distributed control law $\boldsymbol{u}_{i}$ for each agent $i \in \mathcal{V}$ such that, for all edges $(i,j) \in \mathcal{E}$, the following conditions hold simultaneously:
\begin{itemize}
\item \textit{Asymptotic Positional Consensus:} The relative position error converges asymptotically to the prescribed formation offset $\boldsymbol{d}_{ij} \in \mathbb{R}^3$, i.e.,
\begin{equation}
\lim_{t \to \infty} \|\boldsymbol{p}_j(t) - \boldsymbol{p}_i(t) - \boldsymbol{d}_{ij} \| = 0.
\label{eq_prob_pos}
\end{equation}
\item \textit{Bounded Orientational Consensus:} There exists a finite time $T \in \mathbb{R}_{>0}$ such that the relative angular error $\boldsymbol{\theta}_{ij} \in \mathbb{R}^3$, extracted from the rotation matrix $\boldsymbol{R}_i^\top \boldsymbol{R}_j$, satisfies
\begin{equation}
\boldsymbol{\theta}_{ij}(t) \in \Omega_\epsilon \coloneqq \bigl\{\, \boldsymbol{\xi} \in \mathbb{R}^3 \mid \|\boldsymbol{\xi}\| \leq \epsilon_{\mathrm{ori}} \,\bigr\}, \quad \forall t \geq T,
\label{eq_prob_ori}
\end{equation}
where $\epsilon_{\mathrm{ori}} \in \mathbb{R}_{>0}$ is a small tolerance. 
\end{itemize}
\end{problem}
The structure of \cref{prob_split_coordination} is both physically motivated and mathematically non-conservative. 
Regarding condition~\eqref{eq_prob_pos}, asymptotic positional consensus is necessary for accurate trajectory tracking of the shared payload. 
Condition~\eqref{eq_prob_ori} relaxes strict orientational consensus. 
Because during cooperative transport, each manipulator must resolve obstacle avoidance and joint limit constraints, inevitably resulting in transient rotational deviations that cannot be globally suppressed without sacrificing feasibility. 
Moreover, the parameter $\epsilon_{\mathrm{ori}}$ is tunable, where reducing it  approaches rigid coordination, while increasing it improves robustness at the cost of rotational accuracy.

\subsubsection{Control Barrier Functions}
Consider a nonlinear control-affine system in the form:
\begin{equation}
    \dot{\boldsymbol{s}} = f(\boldsymbol{s}) + g(\boldsymbol{s})\boldsymbol{u},
    \label{eq:affine_sys}
\end{equation}
where $\boldsymbol{s} \in \mathcal{D} \subset \mathbb{R}^n$ is the generic 
system state and $\boldsymbol{u} \in \mathcal{U} \subset \mathbb{R}^m$ is 
the control input. 
Let $\mathcal{C} \subset \mathcal{D}$ be a safe set 
defined as the superlevel set of a continuously differentiable function 
$h: \mathcal{D} \to \mathbb{R}$, i.e., $\mathcal{C} \coloneqq \{ \boldsymbol{s} \in \mathcal{D} \mid h(\boldsymbol{s}) \geq 0 \}$ with its boundary $\partial \mathcal{C} \coloneqq \{ \boldsymbol{s} \in 
\mathcal{D} \mid h(\boldsymbol{s}) = 0 \}$. The safety of the system is 
guaranteed if the set $\mathcal{C}$ is \textit{forward invariant}, meaning 
that for any initial condition $\boldsymbol{s}(0) \in \mathcal{C}$, the 
trajectory satisfies $\boldsymbol{s}(t) \in \mathcal{C}$ for all $t \in \mathbb{R}_{\ge 0}$.

A continuously differentiable function $h: \mathcal{D} \to \mathbb{R}$ is 
a CBF for the system \eqref{eq:affine_sys} if 
there exists an extended class-$\mathcal{K}_\infty$ function $\alpha: 
\mathbb{R} \to \mathbb{R}$ such that, for all $\boldsymbol{s} \in \mathcal{D}$,
\begin{equation}
    \sup_{\boldsymbol{u} \in \mathcal{U}} \left[ L_f h(\boldsymbol{s}) 
    + L_g h(\boldsymbol{s})\boldsymbol{u} \right] \geq 
    -\alpha(h(\boldsymbol{s})),
    \label{eq:cbf_condition}
\end{equation}
where $L_f h(\boldsymbol{s}) = \frac{\partial h}{\partial \boldsymbol{s}} 
f(\boldsymbol{s})$ and $L_g h(\boldsymbol{s}) = \frac{\partial h}{\partial 
\boldsymbol{s}} g(\boldsymbol{s})$ denote the Lie derivatives.

\section{Methodology}
\begin{figure*}
    \centering
    \includegraphics[width=1\linewidth]{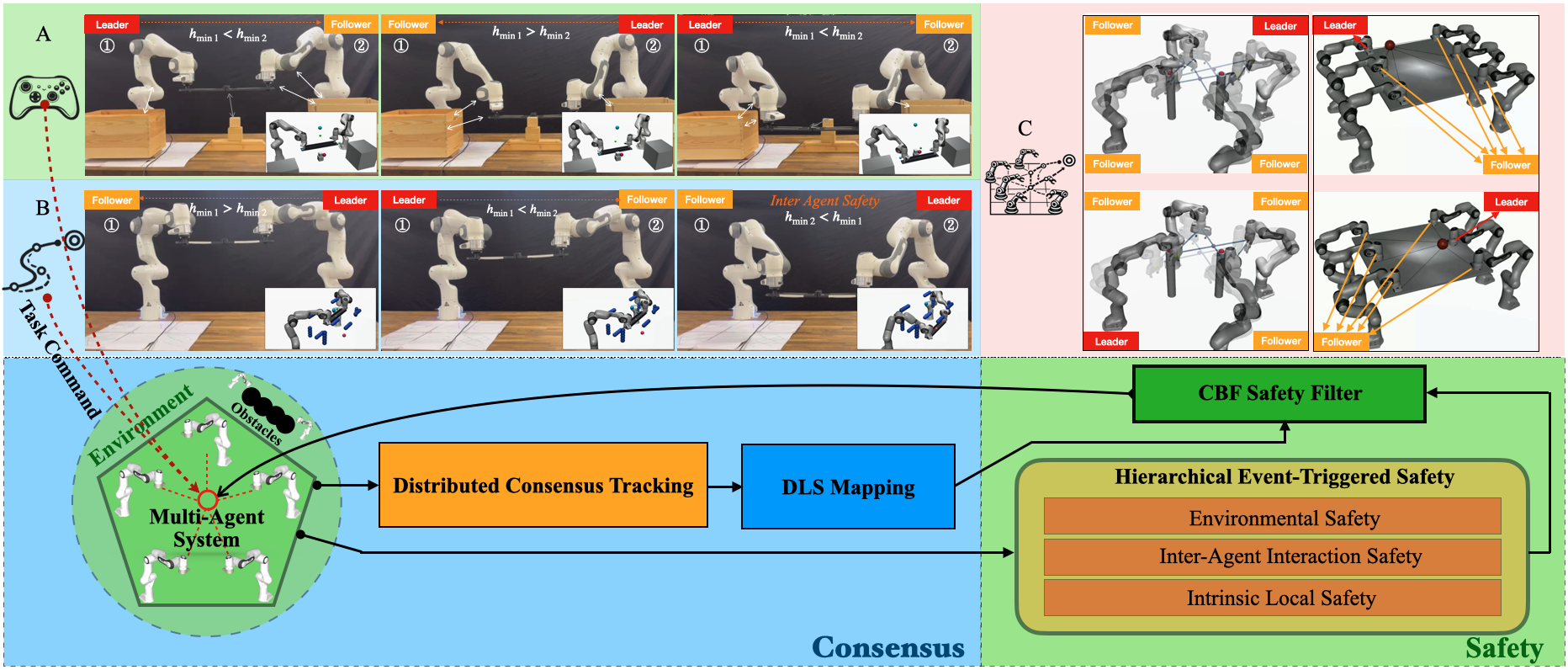}
    \caption{Overview of the proposed framework. The $h_{min}$ function encodes the minimum distance between the robot and the environmental obstacles.}
    \label{fig:overview}
\end{figure*}

\subsection{Distributed Consensus Tracking Protocol}
\label{subsec_consensus_tracking}
To achieve the two-level coordination objective defined in Problem \ref{prob_split_coordination}, we design the virtual task-space control input $\boldsymbol{u}_{i}$ only based on local neighbor information in $\mathcal{N}_i$. 
The control law consists of a control feedback term to maintain geometric formation for high-level planning and a navigational feedforward term to track the group reference for low-level torque control.

For the translational task space, the consensus protocol is designed as a distributed PD-like controller
\begin{align}
\boldsymbol{u}_{i}^{\mathrm{pos}} & = 
  - \sum_{j \in \mathcal{N}_i} a_{ij} \Big[ k_p \big( (\boldsymbol{p}_i - \boldsymbol{p}_j) - \boldsymbol{d}_{ij} \big) + k_d (\dot{\boldsymbol{p}}_i - \dot{\boldsymbol{p}}_j) \Big] \nonumber \\
 & - b_{i} \Big[ k_p (\boldsymbol{p}_i - \boldsymbol{p}_{0} - \boldsymbol{d}_{0,i}) + k_d (\dot{\boldsymbol{p}}_i - \dot{\boldsymbol{p}}_{0})  - \ddot{\boldsymbol{p}}_{0} \Big],
\label{eq:consensus_law_pos}
\end{align}
where $b_{i} =1$ if agent $i$ has direct access to the reference state $\boldsymbol{p}_{0}, \dot{\boldsymbol{p}}_{0}, \ddot{\boldsymbol{p}}_{0}$, and $b_{i}=0$ otherwise. 
Moreover, $k_p, k_d \in \mathbb{R}_{>0}$ are the control gains and $\boldsymbol{d}_{ij}\in\mathbb{R}^3$ is the desired relative displacement from agent $i$ to agent $j$ in task space. 
Furthermore, the angular acceleration controller is
\begin{align}
\boldsymbol{u}_{i}^{\mathrm{ori}} = 
 & - \sum_{j \in \mathcal{N}_i} a_{ij} \Big[ k_{\theta} 
   \tilde{\boldsymbol{\theta}}_{ij} 
   + k_\omega (\boldsymbol{\omega}_i - \boldsymbol{\omega}_j) \Big] 
   \nonumber \\
 & - b_{i} \Big[ k_{\theta} \tilde{\boldsymbol{\theta}}_{0i} 
   + k_\omega (\boldsymbol{\omega}_i - \boldsymbol{\omega}_{0}) 
   - \dot{\boldsymbol{\omega}}_{0} \Big],
\label{eq:consensus_law_ori}
\end{align}
where the world-frame orientation  error $\tilde{\boldsymbol{\theta}}_{ij} = 
\log(\boldsymbol{R}_i \boldsymbol{R}_j^\top)^\vee \in \mathbb{R}^3$,   $\tilde{\boldsymbol{\theta}}_{0i} = \log(\boldsymbol{R}_i \boldsymbol{R}_0^\top)^\vee$ is the orientation error relative to the reference, $\boldsymbol{\omega}_i \in \mathbb{R}^3$ denotes the angular velocity of the EE of agent $i$, expressed in the world frame, and $k_{\theta}, k_\omega \in \mathbb{R}_{>0}$ are control gains. 

Therefore, the total task-space control input is formed as $\boldsymbol{u}_{i} = [(\boldsymbol{u}_{i}^{\mathrm{pos}})^\top, (\boldsymbol{u}_{i}^{\mathrm{ori}})^\top]^\top \in \mathbb{R}^6$.
Note that for agents without direct reference access ($b_i=0$), the reference trajectory information propagates through the network via the coupling terms. 
Therefore, to ensure global convergence, the communication graph $\mathcal{G}$ contains a spanning tree~\cite{Olfati_TAC_2004_Consensus}.

\subsection{Acceleration Mapping via Damped Least Squares (DLS)}
\label{subsec_acc_mapping}
To realize the high-level task-space acceleration command $\boldsymbol{u}_{i}$ on the physical manipulator, we must invert the second-order differential kinematics relation \eqref{eq:task_acc}.
Direct inversion of the Jacobian matrix $\boldsymbol{J}_i(\boldsymbol{q}_i)$ is prone to numerical instability near kinematic singularities, which can result in excessive joint accelerations and torque demands. 
To address this, we formulate the inverse kinematics problem as a Tikhonov-regularized optimization that balances task tracking accuracy against the magnitude of joint acceleration
\begin{align}
\ddot{\boldsymbol{q}}_{i}^{\mathrm{task}}
= \arg\min_{\ddot{\boldsymbol{q}} \in \mathbb{R}^{n}}
\Big(& \frac{1}{2} \|\boldsymbol{J}_i(\boldsymbol{q}_i)\ddot{\boldsymbol{q}} - (\boldsymbol{u}_{i} - \dot{\boldsymbol{J}}_i(\boldsymbol{q}_i, \dot{\boldsymbol{q}}_i)\dot{\boldsymbol{q}}_i)\|_2^2 \nonumber\\
&+ \frac{\lambda^2}{2} \|\ddot{\boldsymbol{q}}\|_2^2 \Big),
\label{eq:acc_dls_opt}
\end{align}
where $\lambda > 0$ is a damping factor. 
The closed-form solution to this problem provides the robust nominal joint acceleration command
\begin{equation}
\ddot{\boldsymbol{q}}_{i}^{\mathrm{task}} = \boldsymbol{J}_i^{\dagger}(\boldsymbol{q}_i) \left( \boldsymbol{u}_{i} - \dot{\boldsymbol{J}}_i(\boldsymbol{q}_i, \dot{\boldsymbol{q}}_i)\dot{\boldsymbol{q}}_i \right),
\label{eq:acc_dls_sol}
\end{equation}
with the damped pseudoinverse defined as
\begin{equation}
\boldsymbol{J}_i^{\dagger} \coloneqq \boldsymbol{J}_i^\top (\boldsymbol{q}_i) \left( \boldsymbol{J}_i(\boldsymbol{q}_i) \boldsymbol{J}_i^\top(\boldsymbol{q}_i) + \lambda^2 \boldsymbol{I} \right)^{-1}.
\end{equation}
This formulation ensures that the generated joint acceleration command remains bounded even when the manipulator passes through or near singular configurations. The term $\dot{\boldsymbol{J}}_i\dot{\boldsymbol{q}}_i$ compensates for the Coriolis and centrifugal effects in the task space, essential for accurate trajectory tracking in dynamic motions.
Finally, the nominal command $\ddot{\boldsymbol{q}}_{i}^{\mathrm{nom}} = \ddot{\boldsymbol{q}}_{i}^{\mathrm{task}} + \boldsymbol{N}_i(\boldsymbol{q}_i)\boldsymbol{\eta}_i$ is passed to the safety filter to enforce strict joint limits and collision constraints, which is detailed in \cref{subsec_safety_architecture}.

\subsection{Hierarchical Event-Triggered Safety Architecture}
\label{subsec_safety_architecture}
\subsubsection{Event-Triggered Environmental Safety}
\label{subsubsec_env_safety}
For manipulators with $n>6$ degrees of freedom, the forward kinematics map $\boldsymbol{f}_i:\mathbb{R}^n\to\mathbb{R}^3
\times\mathrm{SO}(3)$ is not injective.
Consequently, the same EE pose $(\boldsymbol{p}_i, \boldsymbol{R}_i)$ can be realizable by infinitely many joint configurations, which renders the IK map ill-defined without  additional constraints.
This non-uniqueness poses a fundamental difficulty for the leader-centric safety framework: the active leader computes the environmental safety constraint on behalf of the entire formation using only task-space states.
This consensus fails if agents reside on different branches of the self-motion manifold.

To resolve this, we introduce a null-space regulation law that ensures each agent to a common nominal branch of the self-motion manifold throughout task execution. 
Specifically, the auxiliary joint acceleration is set as
\begin{equation}
    \boldsymbol{\eta}_i = k_{\mathcal{N}}
    \left(\boldsymbol{q}_i^{\mathrm{nom}} 
    - \boldsymbol{q}_i\right),
    \quad k_{\mathcal{N}} \in \mathbb{R}_{>0},
    \label{eq:null_space_law}
\end{equation}
where $\boldsymbol{q}_i^{\mathrm{nom}} \in \mathbb{R}^n$ is a shared nominal joint configuration satisfying $\boldsymbol{f}_i(\boldsymbol{q}_i^{\mathrm{nom}}) = (\boldsymbol{p}_i^{\mathrm{init}}, \boldsymbol{R}_0)$.
Since $\boldsymbol{\eta}_i$ enters~\eqref{eq:inverse_acc} only through the null-space projector $\boldsymbol{N}_i$, it has no effect on the task-space dynamics~\eqref{eq:double_integrator_model} and is therefore fully decoupled from the consensus protocol in ~\cref{eq:consensus_law_pos,eq:consensus_law_ori}.

\begin{assumption}[Orientation \& Null-Space Consensus]
\label{asm:fixed_orient}
The manipulation task commences only after the orientation 
consensus law~\eqref{eq:consensus_law_ori} and the 
null-space regulation law~\eqref{eq:null_space_law} have 
jointly driven all agents to practical synchronization. 
Formally, there exists a finite time $T_0\ge 0$ such that 
for all $t\ge T_0$ and all $i\in\mathcal{V}$:
\begin{align}
  &\|\tilde{\boldsymbol{\theta}}_{0i}(t)\|
    \le \epsilon_\theta,
  \quad
  \|\boldsymbol{\omega}_i(t)-\boldsymbol{\omega}_0(t)\|
    \le \epsilon_\omega,
  \nonumber\\
  &\|\boldsymbol{N}_i(\boldsymbol{q}_i)
   \left(\boldsymbol{q}_i(t)
   -\boldsymbol{q}_i^{\mathrm{nom}}\right)\|
    \le \epsilon_{\mathrm{null}},
  \label{eq:consensus_bounds}
\end{align}
where $\epsilon_\theta,\epsilon_\omega,
\epsilon_{\mathrm{null}}\in\mathbb{R}_{>0}$ are 
prescribed tolerance bounds.
\end{assumption}

\cref{asm:fixed_orient} is practically natural in cooperative manipulation, where all agents are commanded to a common pre-grasp configuration $\boldsymbol{q}_i^{\mathrm{nom}}$ in multi-arm task planning. The null-space 
regulation~\eqref{eq:null_space_law} operates purely in the joint space without affecting the EE pose, and can therefore be executed independently by each agent without inter-agent communication.
In the limiting case $\epsilon_\theta=\epsilon_\omega=\epsilon_{\mathrm{null}}=0$, the IK map is locally unique in a neighborhood of $\boldsymbol{q}_i^{\mathrm{nom}}$ away from kinematic singularities, by the Implicit Function Theorem.

Under \cref{asm:fixed_orient}, all agents maintain a 
common nominal branch of the self-motion manifold, 
ensuring that the sphere centers 
$\boldsymbol{c}_{i,k}(\boldsymbol{q}_i)$ are 
kinematically consistent across the formation.
This allows the active leader to evaluate the 
environmental safety constraint on behalf of the entire 
formation.
For each agent, the arm geometry is 
over-approximated by $K_i\in\mathbb{N}$ bounding spheres 
covering all link bodies, with centers
$\boldsymbol{c}_{i,k}(\boldsymbol{q}_i)\in\mathbb{R}^3$, 
$k\in\{1,\dots,K_i\}$, computed via forward kinematics 
and fixed radii $r_{i,k}\in\mathbb{R}_{>0}$.
The signed distance from sphere $(i,k)$ to an obstacle 
point $\boldsymbol{o}\in\mathcal{O}$ is
\begin{equation}
  \phi_{i,k}(\boldsymbol{q}_i,\boldsymbol{o})
  \coloneqq
  \|\boldsymbol{c}_{i,k}(\boldsymbol{q}_i)
    -\boldsymbol{o}\|
  - r_{i,k},
  \label{eq:phi_ik}
\end{equation}
with $\phi_{i,k}\ge 0$ being necessary and sufficient 
for sphere $(i,k)$ to be obstacle-free.

\begin{assumption}[Compact Formation Bound]
\label{asm:compact_form}
There exists a constant 
$\bar{R}_{\mathrm{form}}\in\mathbb{R}_{>0}$ such that 
for all $t\ge T_0$, all $j\in\mathcal{V}$, and all 
$k\in\{1,\dots,K_j\}$,
\begin{equation}
  \min_{k'\in\{1,\dots,K_\ell\}}
  \Big(
    \|\boldsymbol{c}_{j,k}(\boldsymbol{q}_j)
      -\boldsymbol{c}_{\ell,k'}(\boldsymbol{q}_\ell)\|
    - r_{\ell,k'}
    + r_{j,k}
  \Big)
  \le \bar{R}_{\mathrm{form}},
  \label{eq:compact_form}
\end{equation}
where $\boldsymbol{c}_{\ell,k'}$ are the sphere centers 
of the active leader $\ell(t)$, and 
$\bar{R}_{\mathrm{form}}$ can be computed offline from 
the maximum inter-agent separation and the manipulator 
kinematics.
\end{assumption}
\cref{asm:compact_form} states that the entire physical 
volume of every agent (the union of all bounding spheres) 
is enclosed within a bounded envelope around the leader 
body.
It is satisfied by design for compact cooperative 
manipulation formations with bounded inter-agent 
distances enforced by the consensus 
protocol~\eqref{eq:consensus_law_pos}, and is 
consistent with \cref{asm:fixed_orient} which ensures 
kinematic predictability of the sphere centers.

Each agent computes its minimum body clearance:
\begin{equation}
  d_i^{\mathrm{body}}(t)=
  \min_{\boldsymbol{o}\in\mathcal{O}}\,
  \min_{k\in\{1,\dots,K_i\}}
  \phi_{i,k}(\boldsymbol{q}_i(t),\boldsymbol{o})
  \label{eq:d_body}
\end{equation}
using its own joint state and local obstacle 
measurements, and broadcasts the scalar 
$d_i^{\mathrm{body}}$ to all agents over $\mathcal{G}$.
The active leader is the agent with the smallest 
body clearance:
\begin{equation}
  \ell(t)
  = \operatorname*{argmin}_{i\in\mathcal{V}}
    d_i^{\mathrm{body}}(t),
  \label{eq:leader_selection}
\end{equation}
with ties broken by lexicographic order on agent index.
Each agent requires only a single scalar broadcast per 
control step to evaluate~\eqref{eq:leader_selection}.

\begin{lemma}[Leader-to-Formation Safety Transfer]
\label{lem:leader_transfer}
Under \cref{asm:compact_form}, for any 
$\boldsymbol{o}\in\mathcal{O}$, any $j\in\mathcal{V}$, 
and any $k\in\{1,\dots,K_j\}$,
\begin{equation}
  \phi_{j,k}(\boldsymbol{q}_j,\boldsymbol{o})
  \;\ge\;
  d_\ell^{\mathrm{body}}(t) - \bar{R}_{\mathrm{form}}.
  \label{eq:phi_lower_bound}
\end{equation}
Consequently, $d_\ell^{\mathrm{body}}(t)\ge
\bar{R}_{\mathrm{form}}$ implies $\phi_{j,k}\ge 0$ 
for all $(j,k,\boldsymbol{o})$.
\end{lemma}

\begin{proof}
Fix an obstacle point $\boldsymbol{o}\in\mathcal{O}$, an agent $j\in\mathcal{V}$, and a sphere index $k$, and let $k^*\coloneqq\operatorname{argmin}_{k'}\phi_{\ell,k'}(\boldsymbol{q}_\ell,\boldsymbol{o})$ be the active-leader sphere index. By the triangle inequality, $\|\boldsymbol{c}_{j,k}-\boldsymbol{o}\|\ge\|\boldsymbol{c}_{\ell,k^*}-\boldsymbol{o}\|-\|\boldsymbol{c}_{j,k}-\boldsymbol{c}_{\ell,k^*}\|$. Subtracting $r_{j,k}$ and adding and subtracting $r_{\ell,k^*}$, the signed-distance definition~\eqref{eq:phi_ik} gives
\begin{align*}
  \phi_{j,k}
  &\ge
  \underbrace{\big(\|\boldsymbol{c}_{\ell,k^*}-\boldsymbol{o}\|-r_{\ell,k^*}\big)}_{=\,\phi_{\ell,k^*}(\boldsymbol{q}_\ell,\boldsymbol{o})}
  -
  \underbrace{\big(\|\boldsymbol{c}_{j,k}-\boldsymbol{c}_{\ell,k^*}\|-r_{\ell,k^*}+r_{j,k}\big)}_{\le\,\bar{R}_{\mathrm{form}}\text{ by }\eqref{eq:compact_form}}.
\end{align*}
Since $\phi_{\ell,k^*}(\boldsymbol{q}_\ell,\boldsymbol{o})\ge d_\ell^{\mathrm{body}}(t)$ by the minimum-clearance definition~\eqref{eq:d_body}, it follows that $\phi_{j,k}\ge d_\ell^{\mathrm{body}}(t)-\bar{R}_{\mathrm{form}}$, which concludes the proof.
\end{proof}


By \cref{lem:leader_transfer}, we design the CBF for the active leader for environmental safety as follows
\begin{equation}
  h_\ell^{\mathrm{env}}(\boldsymbol{q}_\ell)
=
  d_\ell^{\mathrm{body}}(t)
  - d_{\mathrm{margin}}^{\mathrm{env}}
  \ge 0,
  \label{eq:h_env}
\end{equation}
with the safety margin chosen, so that $d_{\mathrm{margin}}^{\mathrm{env}} > \bar{R}_{\mathrm{form}} + \epsilon_{\mathrm{sens}}$, where $\epsilon_{\mathrm{sens}}\ge 0$ bounds obstacle sensing uncertainty; this design condition is instantiated in the trigger parameters~\eqref{eq:env_param}.

Since $d_\ell^{\mathrm{body}}$ involves a nested 
minimum over $(k,\boldsymbol{o})$, it is generally 
non-differentiable. We apply the active-index method, i.e., at each time instant, identify the minimizing pair
\begin{equation}
  (k^*,\boldsymbol{o}^*)(t)
  \coloneqq
  \operatorname*{argmin}_{
    k\in\{1,\dots,K_\ell\},\,
    \boldsymbol{o}\in\mathcal{O}}
  \phi_{\ell,k}(\boldsymbol{q}_\ell,\boldsymbol{o}),
  \label{eq:active_index}
\end{equation}
breaking ties by a fixed deterministic rule.
At non-smooth instants, Clarke subgradients from $\partial h_\ell^{\mathrm{env}}$ are used, and forward invariance of $\mathcal{S}^{\mathrm{env}}=\{\boldsymbol{q}_\ell:h_\ell^{\mathrm{env}}\ge 0\}$ is preserved by~\cite{Glotfelter_CSL_2017_Nonsmooth}.

Since $h_\ell^{\mathrm{env}}$  in~\eqref{eq:h_env} depends on 
$\boldsymbol{q}_\ell$ and the closed-loop system 
exhibits double-integrator task-space 
dynamics~\eqref{eq:double_integrator_model}, 
$h_\ell^{\mathrm{env}}$ has relative degree two 
with respect to $\boldsymbol{u}_\ell$.
Define
\begin{equation}
  b_\ell^{\mathrm{env}}=
  \dot{h}_\ell^{\mathrm{env}}
  +\alpha_1^{\mathrm{env}}(h_\ell^{\mathrm{env}}),
  \quad\alpha_1^{\mathrm{env}}\in\mathcal{K}_\infty,
  \label{eq:b_env}
\end{equation}
where $\alpha_{1}\in\mathcal{K}_{\infty}$ is a class $\mathcal{K}$-infinity and $\dot{h}_\ell^{\mathrm{env}}
=(\partial h_\ell^{\mathrm{env}}/
  \partial\boldsymbol{q}_\ell)
 \dot{\boldsymbol{q}}_\ell$.
The second-order CBF constraint
\begin{equation}
  \dot{b}_\ell^{\mathrm{env}}
  +\alpha_2^{\mathrm{env}}(b_\ell^{\mathrm{env}})
  \ge 0,
  \quad\alpha_2^{\mathrm{env}}\in\mathcal{K}_{\infty},
  \label{eq:cbf_env}
\end{equation}
where $\alpha_{2}\in\mathcal{K}_{\infty}$, is appended to the leader's QP if and only if 
$\mathcal{E}_\ell^{\mathrm{env}}(t)=1$.
Let $d_{\mathrm{alert}}^{\mathrm{env}}>
d_{\mathrm{margin}}^{\mathrm{env}}$ and 
$\delta_{\mathrm{hyst}}^{\mathrm{env}}>0$ be the 
alert distance and hysteresis margin, respectively. 
Define the scalar trigger function
\begin{equation}
  f_\ell^{\mathrm{env}}(t)
  \coloneqq
  d_\ell^{\mathrm{body}}(t)
  - d_{\mathrm{alert}}^{\mathrm{env}}.
  \label{eq:f_env}
\end{equation}
The trigger function is designed as 
\begin{equation*}
  \mathcal{E}_\ell^{\mathrm{env}}(t^+) =
  \begin{cases}
    1, & \mathcal{E}_\ell^{\mathrm{env}}(t)=0
         \;\wedge\; f_\ell^{\mathrm{env}}(t)\le 0,
         \\[4pt]
    0, & \mathcal{E}_\ell^{\mathrm{env}}(t)=1
         \;\wedge\; f_\ell^{\mathrm{env}}(t)
           \ge\delta_{\mathrm{hyst}}^{\mathrm{env}},
         \\[4pt]
    \mathcal{E}_\ell^{\mathrm{env}}(t),
       & \text{otherwise,}
  \end{cases}
  \label{eq:env_trigger}
\end{equation*}
where $\delta_{\mathrm{hyst}}^{\mathrm{env}}$ prevents 
chattering near the alert boundary.
The design parameters satisfy
\begin{equation}
  d_{\mathrm{alert}}^{\mathrm{env}}
  \ge d_{\mathrm{margin}}^{\mathrm{env}}
  > \bar{R}_{\mathrm{form}}+\epsilon_{\mathrm{sens}},
  \label{eq:env_param}
\end{equation}
ensuring the trigger fires \emph{before} 
$h_\ell^{\mathrm{env}}$ becomes negative, thereby 
guaranteeing feasibility of~\eqref{eq:cbf_env} 
upon activation.
Leader switches occur when~\eqref{eq:leader_selection} 
changes value. At each switch instant $t_s$, the 
incoming leader $\ell^+$ must satisfy
\begin{equation}
  h_{\ell^+}^{\mathrm{env}}
  \!\left(\boldsymbol{q}_{\ell^+}(t_s)\right)\ge 0,
  \label{eq:switch_feasibility}
\end{equation}
ensuring constraint~\eqref{eq:cbf_env} is initially 
satisfiable for the new leader.
The trigger state $\mathcal{E}_{\ell^+}^{\mathrm{env}}$ 
is reset to zero and immediately re-evaluated using 
$\boldsymbol{q}_{\ell^+}(t_s)$.
Stability under switching follows from the dwell time 
arguments of~\cite{liberzon2003switching}, provided 
the minimum dwell time $\tau_{\min}>\tau^*$, where 
$\tau^*>0$ is determined by the CBF decay rate 
$\alpha_1^{\mathrm{env}}$. 
The necessary condition of $\tau^*$ is not considered in this paper and is left for future work.

\subsubsection{Event-Triggered Inter-Agent Interaction Safety}
\label{subsubsec_inter_safety}

Inter-agent safety is enforced by ensuring that the
EE of every agent $i$ maintains a safe distance
from the EEs of all neighboring agents
$j\in\mathcal{N}_i$.
For each pair $(i,j)$ with $i<j$, define the pairwise
EE distance barrier
\begin{equation}
  h_{ij}^{\mathrm{inter}}
  \coloneqq
  \|\boldsymbol{p}_i(\boldsymbol{q}_i)
    - \boldsymbol{p}_j(\boldsymbol{q}_j)\|
  - d_{\mathrm{margin}}^{\mathrm{inter}},
  \label{eq:h_inter}
\end{equation}
where $d_{\mathrm{margin}}^{\mathrm{inter}}>0$ is the
minimum allowable EE separation.
Since $h_{ij}^{\mathrm{inter}}$ depends on both
$\boldsymbol{q}_i$ and $\boldsymbol{q}_j$, it has relative
degree two w.r.t. the joint accelerations of both agents.
Let $d_{\mathrm{alert}}^{\mathrm{inter}}>0$ be an alert
horizon and $\delta_{\mathrm{hyst}}^{\mathrm{inter}}>0$
a hysteresis width.
Define the pairwise trigger function
\begin{equation}
  g_{ij}(t)
  \coloneqq
  \|\boldsymbol{p}_i(t)-\boldsymbol{p}_j(t)\|
  - d_{\mathrm{alert}}^{\mathrm{inter}},
  \label{eq:g_inter}
\end{equation}
and the symmetric trigger function is
\begin{equation*}
  \mathcal{E}_{ij}^{\mathrm{inter}}(t^+) =
  \begin{cases}
    1, & \mathcal{E}_{ij}^{\mathrm{inter}}(t)=0
         \;\wedge\; g_{ij}(t)\le 0,\\[4pt]
    0, & \mathcal{E}_{ij}^{\mathrm{inter}}(t)=1
         \;\wedge\; g_{ij}(t)\ge
           \delta_{\mathrm{hyst}}^{\mathrm{inter}},\\
    \mathcal{E}_{ij}^{\mathrm{inter}}(t),
         & \text{otherwise,}
  \end{cases}
  \label{eq:inter_trigger}
\end{equation*}
which is evaluated identically by both agents $i$ and $j$
given that $g_{ij}=g_{ji}$.
The design parameters satisfy
\begin{equation}
  d_{\mathrm{alert}}^{\mathrm{inter}}
  \ge d_{\mathrm{margin}}^{\mathrm{inter}}
  > 0,
  \label{eq:inter_margin}
\end{equation}
ensuring the constraint is activated before the barrier becomes negative.
When $\mathcal{E}_{ij}^{\mathrm{inter}}=0$, no state
information need be exchanged over edge $(i,j)$, since the
constraint is inactive.
Communication is required only when two agents are in close
proximity, substantially reducing network load in
sparsely occupied workspaces.

\subsubsection{Intrinsic Local Safety}
\label{subsubsec_local_safety}

Each agent $i\in\mathcal{V}$ enforces the following 
intrinsic safety constraints at every control step,
independently of the event-triggered mechanism.

\begin{itemize}

\item {Joint position limits:}
For each joint $k$, define
$h_{i,k}^{q,+} \coloneqq q_{i,k} - q_{\min,i,k}$
and
$h_{i,k}^{q,-} \coloneqq q_{\max,i,k} - q_{i,k}$.
Both have relative degree two w.r.t.\
$\ddot{\boldsymbol{q}}_i$ and are enforced via 
the HO-CBF condition
$\dot{b}_{i,k}^{q,\pm}+\alpha_2(b_{i,k}^{q,\pm})
\ge 0$,
where $b_{i,k}^{q,\pm}\coloneqq
\dot{h}_{i,k}^{q,\pm}+\alpha_1(h_{i,k}^{q,\pm})$,
$\alpha_1,\alpha_2\in\mathcal{K}_\infty$.

\item {Joint velocity limits:}
Define $h_i^{v}\coloneqq
\dot{q}_{\max,i}^{2}-\|\dot{\boldsymbol{q}}_i\|^{2}$.
Since $h_i^v$ has relative degree one, it is 
enforced by the standard CBF condition
$\dot{h}_i^{v}+\alpha_1(h_i^{v})\ge 0$.

\item {Torque limits:}
The constraint
$\boldsymbol{\tau}_{\min,i}\le\boldsymbol{\tau}_i
\le\boldsymbol{\tau}_{\max,i}$
is imposed as a box constraint on the QP decision 
variable via the inverse dynamics 
law~\eqref{eq:inverse_dynamics_ctrl}.

\item {Self-collision avoidance:}
Let $d_i^{\mathrm{self}}(\boldsymbol{q}_i)$ 
denote the minimum distance between any 
non-adjacent link pair of agent $i$, and let
$\delta^{\mathrm{self}}\in\mathbb{R}_{>0}$ be a 
prescribed safety margin.
Define
$h_i^{\mathrm{sc}}\coloneqq
d_i^{\mathrm{self}}(\boldsymbol{q}_i)
-\delta^{\mathrm{self}}$,
enforced via the same HO-CBF condition as the 
joint position barriers, using Clarke subgradients 
at non-smooth 
instants~\cite{Glotfelter_CSL_2017_Nonsmooth}.
\end{itemize}
For every position-dependent barrier 
$h_{i,k}(\boldsymbol{q}_i)$, the HO-CBF condition 
reduces to the affine constraint in 
$\ddot{\boldsymbol{q}}_i$
\begin{equation}
  \frac{\partial h_{i,k}}{\partial\boldsymbol{q}_i}
  \ddot{\boldsymbol{q}}_i
  \ge
  -\frac{\partial\dot{h}_{i,k}}
        {\partial\boldsymbol{q}_i}
  \dot{\boldsymbol{q}}_i
  - \alpha_1'(h_{i,k})\dot{h}_{i,k}
  - \alpha_2(b_{i,k}),
  \label{eq:cbf_affine}
\end{equation}
which is the form directly appended to the unified 
QP in \cref{subsubsec_unified_qp}.

\subsubsection{Unified QP Safety Filter}
\label{subsubsec_unified_qp}

At each control step, every agent $i\in\mathcal{V}$ 
computes a safe joint acceleration by solving the 
following quadratic program:
\begin{equation}
  \ddot{\boldsymbol{q}}_{\mathrm{safe},i}
  =
  \operatorname*{arg\,min}_{
    \boldsymbol{z}\in\mathcal{Z}_i^{\tau}}
  \;\tfrac{1}{2}
  \left\|\boldsymbol{z}
    - \ddot{\boldsymbol{q}}^{\mathrm{nom}}_{i}
  \right\|^2
  \label{eq:safety_qp}
\end{equation}
\begin{equation*}
  \text{s.t.}\quad
  \dot{b}_{i,k}(\boldsymbol{q}_i,
    \dot{\boldsymbol{q}}_i,\boldsymbol{z})
  + \alpha_2(b_{i,k}) \ge 0,
~
  \forall\,k\in
  \mathcal{K}_{\mathrm{int},i}
  \cup
  \mathcal{K}_{\mathrm{active},i}(t),
\end{equation*}
where $\mathcal{Z}_i^{\tau}$ is the torque-feasible set defined the torque limits, $\mathcal{K}_{\mathrm{int},i}$ is the 
    always-active intrinsic constraint index set defined in \cref{subsubsec_local_safety},
 $\mathcal{K}_{\mathrm{active},i}(t)$ is the 
    event-triggered active constraint index set, 
    assembled at each control step from the 
    trigger state machines defined 
    in \cref{eq:env_trigger} and the inter-agent 
    safety triggers in \cref{eq:inter_trigger}.
Each constraint in~\eqref{eq:safety_qp} is affine in 
$\boldsymbol{z}$, since
\begin{equation}
  \dot{b}_{i,k}
  =
  {
    \frac{\partial h_{i,k}}{\partial\boldsymbol{q}_i}
  }
  \boldsymbol{z}
  +
  {
    \frac{\partial\dot{h}_{i,k}}
         {\partial\boldsymbol{q}_i}
    \dot{\boldsymbol{q}}_i
  },
  \label{eq:bdot_affine}
\end{equation}
and thus~\eqref{eq:safety_qp} is a strictly convex QP.

Since $d_{\mathrm{alert}}^{\mathrm{env}}
\ge d_{\mathrm{margin}}^{\mathrm{env}}$ 
by~\eqref{eq:env_param}, each event-triggered CBF 
constraint is activated only while the corresponding 
barrier $h_{i,k}>0$, i.e., strictly inside the safe 
set.
At such instants, $\boldsymbol{z}=
\ddot{\boldsymbol{q}}^{\mathrm{nom}}_{i}$ 
is a feasible point whenever the nominal input 
satisfies the barrier condition, which holds by 
construction of the alert threshold.
If needed, feasibility can be universally guaranteed 
by augmenting~\eqref{eq:safety_qp} with a slack 
variable $\delta_i\ge 0$ and a penalty term 
$p\,\delta_i^2$ in the objective.

\section{EXPERIMENTS AND RESULTS}
In this section, we evaluate four methods in both real-world experiments and simulations under identical conditions and desired references for quantitative analysis: Ours denotes the proposed HET-CBF framework; Distributed CBF (D-CBF)~\cite{Mestres_RAL_2024_Distributed} refers to a baseline where all agents independently solve a CBF-QP safety filter at every control update to enforce both obstacle avoidance and inter-arm collision avoidance; Centralized NMPC~\cite{Hu_IROS_2021_NMPC} is a centralized nonlinear model predictive controller that optimizes the nominal control sequence; Centralized MPPI~\cite{williams_arxiv_2015_modelpredictivepathintegral} is a centralized model predictive path integral controller that generates the nominal control sequence.

\subsection{Real World Dual-arm Task}
\label{subsec_realworld}
\subsubsection{Experimental Description}
The real-world platform consists of two Franka Emika Panda manipulators, each rigidly grasping one end of a transported link to emulate a cooperative transportation scenario. 
Notably, to ensure collision avoidance of the transported object, each agent treats its end-effector and the proximal half of the link as a unified rigid body, enclosed by a single bounding capsule for safety constraint evaluation. 
The formation is regulated by a fixed desired relative pose $(\boldsymbol{d}_{12}^{\mathrm{des}}, \tilde{\boldsymbol{\theta}}_{12}^{\mathrm{des}})$ between the two EEs, where $\boldsymbol{d}_{12}^{\mathrm{des}}=[0,\ 0.50,\ 0]^\top$~m and $\tilde{\boldsymbol{\theta}}_{12}^{\mathrm{des}} =[0.32,-0.32,1.95]^\top$ rad. 
The minimal dwell time $\tau_{\min}=0.15$~s,  and the hysteresis margin is set to {$\delta^{env}_{\mathrm{hyst}} = 2\times10^{-4}$~m}.
The consensus control gains are set as $k_p=3, k_d=1, k_\theta =3, k_\omega=0.5$. 
The joint-velocity limit $\dot{q}_{\max}=2$~rad/s, a DLS factor $\lambda=0.05$, and the safety margin $d_{\mathrm{safe}}=0.01$~m. 
The position and torque limits and self collision configuration are set as the same physical limitation of Franka arms.
The task-space reference is defined by the position center of the a transported link $\boldsymbol{p}_c = 1/2(\boldsymbol{p}_1 + \boldsymbol{p}_2)$, whose the target position is set as $\boldsymbol{p}_g = [0.45,0.0,0.3]$. It is generated by a straight-line trajectory in task space obtained by linearly interpolating from the initial $\boldsymbol{p}_c = [0.24, 0.00, 0.63]^\top$ to $\boldsymbol{p}_g$ and only sent to the active leader defined by $\ell(t)$.

\subsubsection{Results}
In \cref{fig_consensus_error}, we compare the consensus position error defined as $E_p=1/N\sum_{i,j \in \mathcal{V}}\|\boldsymbol{p}_j(t) - \boldsymbol{p}_i(t) - \boldsymbol{d}_{ij} \|$ and rotation error denoted by $E_{\theta}= 1/N\sum_{i,j \in \mathcal{V}}\|\tilde{\boldsymbol{\theta}}_{ij} \|$ across four methods. 
The shaded intervals indicate which leader is active under the proposed HET-CBF. HET-CBF achieves the smallest $E_p$ for most of the time, staying below 0.005. 
In contrast, NMPC settles at a relatively large steady-state error, MPPI shows pronounced oscillations, and the distributed CBF performance deteriorates markedly after 200 steps.
For the rotation error $E_{\theta}$, HET-CBF achieves the lowest values overall. MPPI also keeps the error at a low level, but it exhibits noticeable chattering. 
Meanwhile, distributed CBF exhibits a peak of about 0.3, and NMPC reduces the error only gradually. 
Notably, under HET-CBF, $E_{\theta}$ increases around 100 and 700 steps due to obstacle avoidance behavior. This is expected because, in the prioritized objective of the control design (cf. \cref{prob_split_coordination}), position consensus is enforced as the primary goal, while orientation tracking is treated as secondary.
\begin{figure}[t]
    \centering
    \includegraphics[width=1\linewidth]{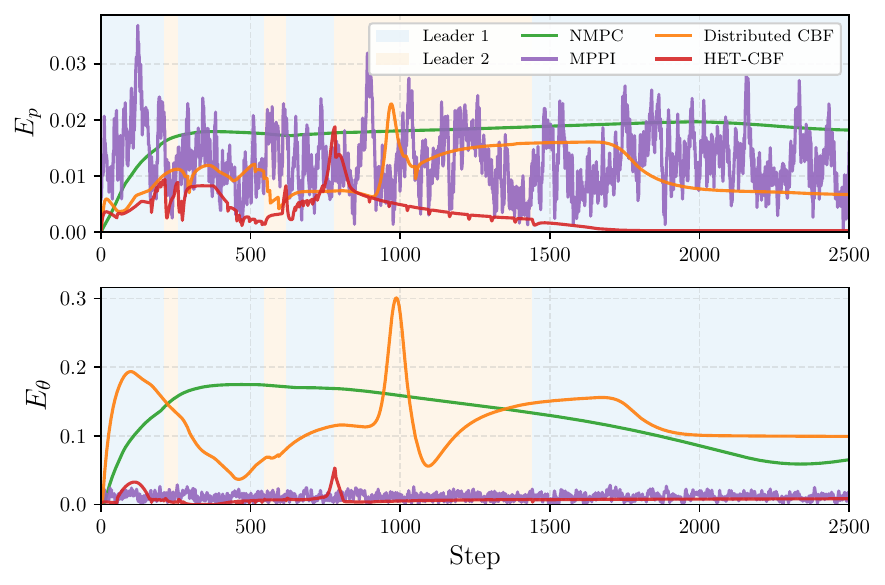}
    \caption{Comparison of formation position error and orientation error across methods and leader active duration of HET-CBF (shaded regions).}
    \label{fig_consensus_error}
\end{figure}

\subsection{Monte Carlo Experiment}
\subsubsection{Experimental Description}
To assess robustness, we perform a Monte Carlo study with $N_{\mathrm{trial}}=20$ trials of two-arm simulations conducted in MuJoCo in a workspace with the same static sphere-based obstacles in \cref{subsec_realworld}. 
Moreover, unless otherwise stated, the remaining controller parameters and robotic constraints are the same as in the real-world experiments in \cref{subsec_realworld}. 
In each trial, we apply small random perturbations to the initial condition, the Cartesian goal position, and the obstacle-sphere centers, and reuse the same randomized trial configuration across all compared methods to enable a paired and fair comparison. 
Specifically, the formation-center configuration is perturbed by $\Delta q_c\sim\mathcal{U}[-0.04,\,0.04]$~m, the goal position is perturbed per axis by $\Delta p_{g}\sim\mathcal{U}[-0.02,\,0.02]$~m, and each obstacle center is perturbed per axis by $\Delta o_k\sim\mathcal{U}[-0.02,\,0.02]$~m.  

\subsubsection{Results}
In terms of tracking accuracy, HET-CBF achieves the lowest aggregated formation position error $E_p$ and orientation error $E_\theta$ as shown in \cref{fig_mc}, outperforming NMPC by nearly an order of magnitude and surpassing both MPPI and D-CBF by a substantial margin. 
Regarding computational efficiency, HET-CBF reduces the per-step solve time to approximately 8 ms over 19$\times$ faster than NMPC and 25$\times$ faster than MPPI, validating the scalability of the event-triggered QP architecture for real-time deployment. 
Furthermore, HET-CBF achieves the shortest task completion time with a median of approximately 5.3 s, while D-CBF incurs the highest latency, demonstrating that the hierarchical leader-switching strategy in HET-CBF not only enforces safety but also preserves task efficiency.
\begin{figure}
    \centering
    \includegraphics[width=1\linewidth]{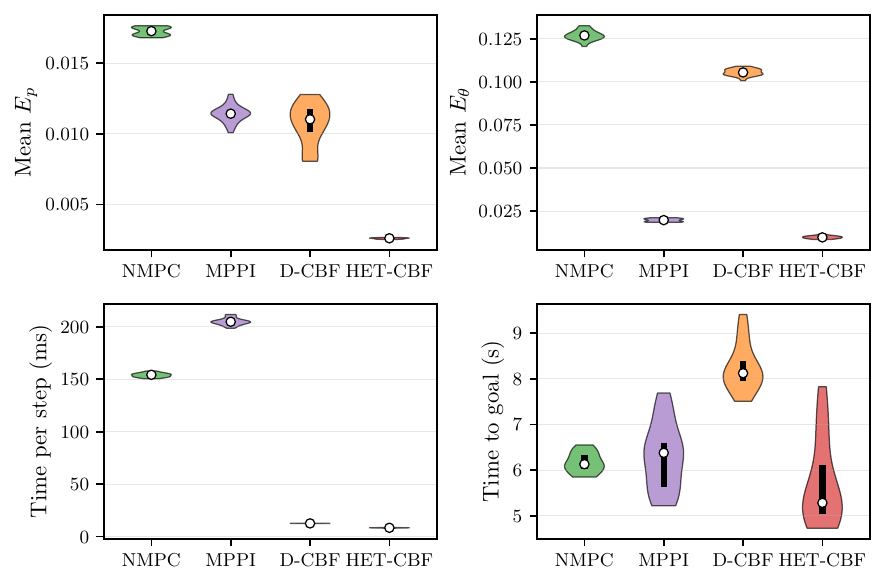}
    \caption{Monte Carlo evaluation comparing formation position error, orientation error, computational time, and task completion time across methods.}
    \label{fig_mc}
\end{figure}

\subsection{Multi-arm Task}
\subsubsection{Experimental Description}
We further evaluate scalability in simulation using a multi-arm cooperative task in MuJoCo. In this scenario, a single dynamic spherical obstacle with an angular velocity of 0.2 rad/s moves along a circular trajectory around the four manipulators, repeatedly approaching the formation and creating time-varying risk. 
In our method, communication follows a ring topology, where $\mathcal{E} = \{(1,2),\,(2,3),\,(3,4),\,(4,1)\}$. 
Unless otherwise stated, the remaining controller parameters are the same as in the real-world experiments in \cref{subsec_realworld}.
The team-level task variable is defined as the formation center $\boldsymbol{p}_c \coloneqq \tfrac{1}{4}\sum_{i=1}^{4}\boldsymbol{p}_i$ = [0.50,0.00,0.62], and the formation is generated by the initial joint space $[0.0, -0.57, 0.0, -2.14, 0.0, 1.57, 0.79]$ and the base positions are set to $\{[0,\,\pm0.50,\,0]^\top,\ [\pm1.00,\,\pm0.50,\,0]^\top\}$. 
The dynamic obstacle follows a circular path with radius $r_{\mathrm{circ}}=0.30$ m, center and angular speed $\omega_{\mathrm{obs}}=0.2$ rad/s.
We first evaluate the method under human-in-the-loop control, where a user teleoperates the active leader using a world-frame end-effector velocity command of $0.6$~m/s, while the remaining agents maintain the formation via distributed consensus. We record the minimum safety barrier value $h_{\min}(t)$, defined as the environmental barrier $h_i^{\mathrm{env}}(t)$ of the arm that is closest to the external spherical obstacle, i.e., $h_{\min}(t)\coloneqq \min_{i\in\mathcal{V}} h_i^{\mathrm{env}}(t)$ during an execution.
Second, we conduct station-keeping experiments in which the multi-arm system maintains its nominal initial pose while avoiding a dynamic spherical obstacle. We quantitatively benchmark the proposed framework against baseline methods in terms of formation accuracy, orientation error, computational cost, and safety compliance (i.e., whether all agents successfully avoid obstacles).

\subsubsection{Results}

\begin{figure}
    \centering
    \includegraphics[width=0.95\linewidth]{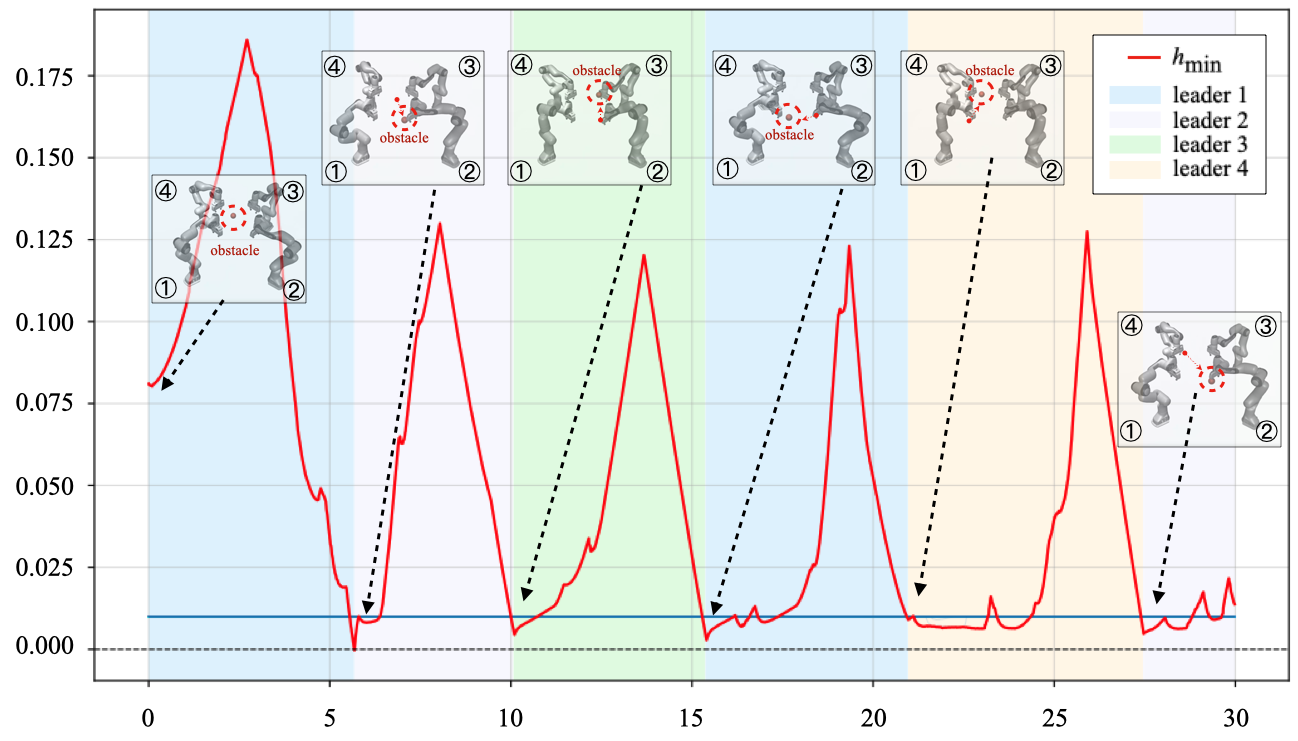}
    \caption{ 
    Visualization of $h_{\min}(t)$ and leader active duration in four-arm scenario.
    }
    \label{fig:enter-label}
\end{figure}

\begin{table}[t]
\centering
\caption{Performance comparison in four-arm simulation.}
\label{tab:four_methods_summary_meanpmrange}
\setlength{\tabcolsep}{4pt}
\renewcommand{\arraystretch}{1.12}
\resizebox{\columnwidth}{!}{%
\begin{tabular}{lcccc}
\toprule
Method & Safety & Time per step (ms) & $E_p$ & $E_{\theta}$  \\
\midrule
HET-CBF   & \cmark & \textbf{3.06$\pm$0.51}   & \textbf{0.0080$\pm$0.0092} & \textbf{0.030$\pm$0.072} \\
D-CBF           & \cmark & 3.32$\pm$0.37            & {0.0110}$\pm$0.0068          & 0.039$\pm$0.086          \\
NMPC       & \xmark & 42.03$\pm$7.56            & {0.0360}$\pm${0.0140}         & {0.130}$\pm$0.090         \\
MPPI       & \xmark & {259.20}$\pm$13.90            & {0.0100}$\pm$0.0092         & 0.047$\pm$0.062           \\
\bottomrule
\end{tabular}}
\end{table}
In \cref{fig:enter-label}, when teleoperating the active leader, the formation motion together with the circularly moving obstacle causes the risk to shift among different arms, leading to periodic drops in $h_{\min}(t)$. The blue curve indicates the prescribed safety margin $d_{\mathrm{safe}}$. Meanwhile, the active leader rotates among the manipulators and is always assigned to the arm currently closest to the obstacle, ensuring safety throughout the execution, as indicated by $h_{\min}(t)>0$ at all times.
And as shown in \cref{tab:four_methods_summary_meanpmrange}, the proposed framework achieves the best performance across all metrics, where the value of computational step time, $E_p$ and $E_{\theta}$ are aggregated for all agents.
Note that a method is marked safe ($\checkmark$) if and only if the
minimum safety barrier remains positive throughout the entire execution,
i.e., $h_{\min}(t)=\min_{i\in\mathcal{V}} h^{env}_i(t) > 0$ and vice versa.
Only Ours and D-CBF satisfy safety constraints throughout the task, while NMPC and MPPI fail to avoid obstacles. Ours attains the lowest per-step solve time, the smallest position error, and the lowest rotational error, showing that the HET architecture achieves real-time feasibility, strict safety compliance, and superior tracking precision over all baseline methods.

\section{CONCLUSION}
This paper presented a distributed hierarchical framework for safe cooperative manipulation with multiple manipulators. 
To enforce strict safety constraints without excessive computational overhead, we developed a three-layer event-triggered architecture that integrates CBFs with a risk-aware leader switching strategy, efficiently distributing safety-critical computation across the network. 
Real-world experiments with two Franka manipulators validated the framework in real-world static obstacle scenarios, while comprehensive Monte Carlo simulations and four-arm scenarios confirmed its scalability and robustness under dynamic obstacle environments. 
Compared to state-of-the-art baselines, the proposed framework maintains strict constraint satisfaction and high-precision formation tracking while substantially reducing QP solve times and communication cost.








\bibliographystyle{ieeetr}
\bibliography{ref}  
\end{document}